\documentclass[letterpaper, 10 pt, conference]{ieeeconf} 
\IEEEoverridecommandlockouts
\usepackage{graphicx}
\usepackage{amsmath}
\usepackage{amssymb}
\usepackage[T1]{fontenc}
\usepackage{textcomp}
\usepackage{multirow}
\usepackage{multicol}
\usepackage{arydshln}
\usepackage{subfig}
\usepackage{wrapfig}
\usepackage{balance}

\PassOptionsToPackage{hyphens}{url}\usepackage{hyperref}
\newtheorem{lemma}{Lemma}
\newcommand\defeq{\mathrel{\stackrel{\makebox[0pt]{\mbox{\normalfont\scriptsize def}}}{:=}}}

\usepackage{tikz}
\newcommand\copyrighttext{%
  \footnotesize © 2024 IEEE.  Personal use of this material is permitted.  Permission from IEEE must be obtained for all other uses, in any current or future media, including reprinting/republishing this material for advertising or promotional purposes, creating new collective works, for resale or redistribution to servers or lists, or reuse of any copyrighted component of this work in other works.}
\newcommand\copyrightnotice{%
\begin{tikzpicture}[remember picture,overlay]
\node[anchor=south,yshift=10pt] at (current page.south) {\fbox{\parbox{\dimexpr\textwidth-\fboxsep-\fboxrule\relax}{\copyrighttext}}};
\end{tikzpicture}%
}

\title{\LARGE \bf
EC-IoU: Orienting Safety for Object Detectors \\via Ego-Centric Intersection-over-Union
}

\author{Brian Hsuan-Cheng Liao$^{\dagger\ddagger}$, Chih-Hong Cheng$^{\mathsection\sharp}$, Hasan Esen$^{\dagger}$, Alois Knoll$^{\ddagger}$%
\thanks{$^{\dagger}$DENSO AUTOMOTIVE Deutschland GmbH, Germany} %
\thanks{$^{\ddagger}$Technical University of Munich, Germany} %
\thanks{$^{\mathsection}$Chalmers University of Technology, Sweden} %
\thanks{$^{\sharp}$University of Gothenburg, Sweden} %
\thanks{Correspondence to \tt\small {h.liao}@eu.denso.com}
}

\begin{document}

\maketitle
\copyrightnotice

\begin{abstract}

This paper presents Ego-Centric Intersection-over-Union (EC-IoU), addressing the limitation of the standard IoU measure in characterizing safety-related performance for object detectors in navigating contexts. Concretely, we propose a weighting mechanism to refine IoU, allowing it to assign a higher score to a prediction that covers closer points of a ground-truth object from the ego agent's perspective. The proposed EC-IoU measure can be used in typical evaluation processes to select object detectors with better safety-related performance for downstream tasks. It can also be integrated into common loss functions for model fine-tuning. While geared towards safety, our experiment with the KITTI dataset demonstrates the performance of a model trained on EC-IoU can be better than that of a variant trained on IoU in terms of mean Average Precision as well.

\end{abstract}

\section{Introduction}

Object detection is an essential function in robot perception and navigation. Thanks to emerging learning-based algorithms and pipelines, object detectors have achieved unprecedented performance and have been applied in many application domains~\cite{wu2020recent}. Nevertheless, some of the applications, especially those bearing safety criticality, seem to meet certain challenges when it comes to scaling, e.g., mass production and broad deployment of highly automated vehicles in the autonomous driving industry. According to industrial standards such as ISO 21448 (a.k.a. SOTIF)~\cite{iso21448} and ANSI/UL 4600~\cite{ul4600}, one crucial yet seemingly lacking factor is the implementation of leading safety-related performance indicators. Having observed the potential gap, we develop a safety-oriented measure for object detectors that can better reflect the notion of safety and, thereby, alleviate the challenges.

Typically, the Intersection-over-Union (IoU) measure is used for comparing model predictions against ground truths (taken as objects)~\cite{padilla2020survey}. It provides a good indication of the model performance in common scene understanding applications. However, IoU focuses only on the absolute position of the ground truth (i.e., being object-centric) and may be limited when the relative position between the ground truth and the ego is crucial (e.g., in a driving context). Additionally, studies have suggested that state-of-the-art object detectors usually saturate at IoU's around~$0.7$\textasciitilde$1$~\cite{he2022alphaiou}. It is sometimes hard to further differentiate or improve these object detectors. Hence, there is a need for a more fine-grained indicator to distinguish such saturating models and predict which will incur fewer safety concerns during system operations.

\begin{figure}[t]
    \centering
    \includegraphics[width=0.65\linewidth]{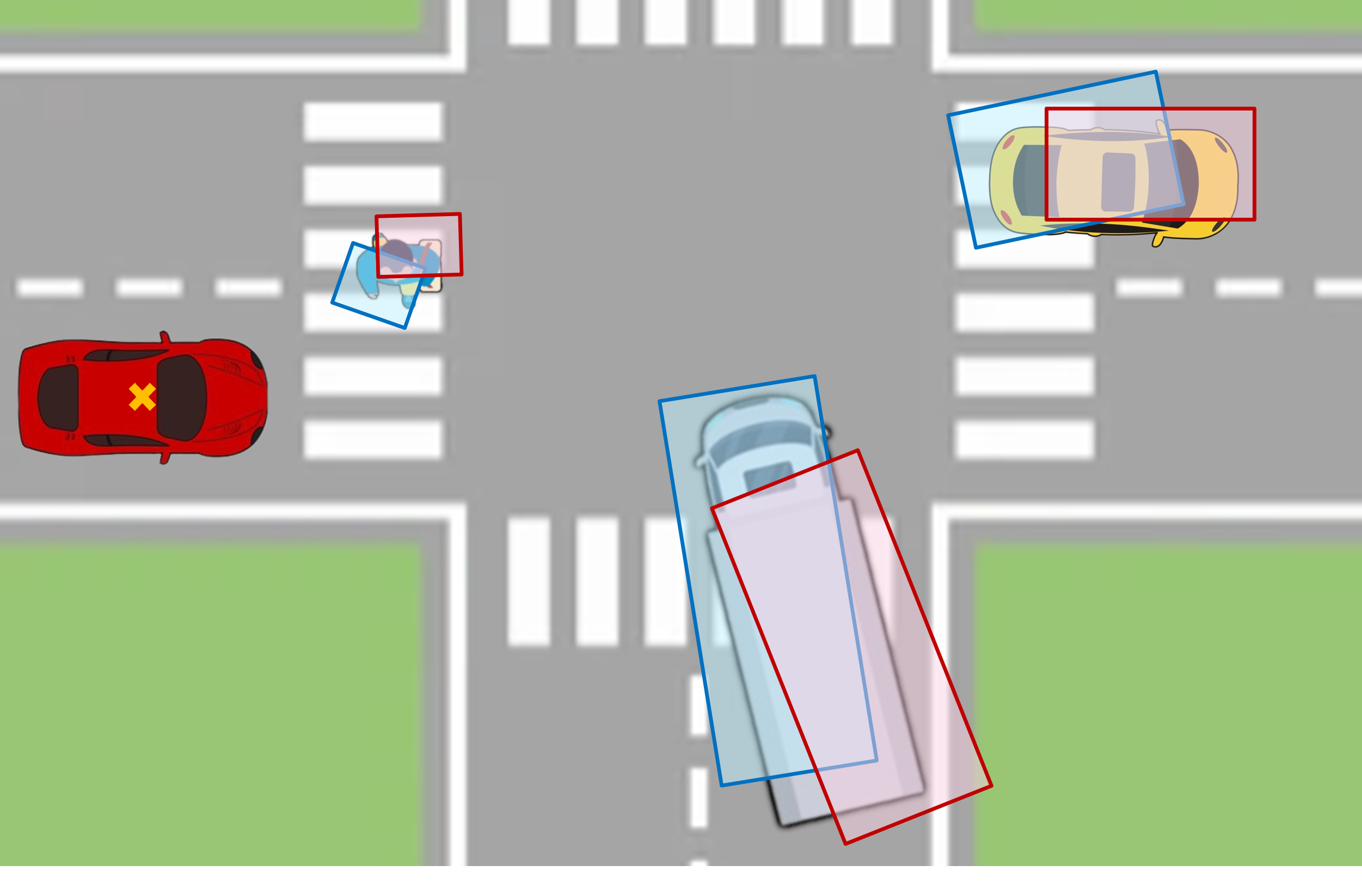}
    \caption{A diagram showing our motivation. All blue and red predictions have an IoU of roughly~$0.7$. However, the blue ones should be prioritized to avoid potential collisions at the front of the objects from the angle of the (red) ego car.}
    \label{fig:concept}
\end{figure}

Our key contribution, thus, is the proposal of a refined measure, Ego-Centric IoU (EC-IoU), taking into account the ego's position when assessing a prediction against its ground truth. By doing so, the potential safety (or danger) of the ego and the object can be better reflected. To illustrate, given two predictions in the vicinity of ground truth, the one that covers a portion of the ground truth closer to the ego agent should be assessed as better. Fig.~\ref{fig:concept} depicts the concept. Technically, we start with a weighting function that assigns different levels of importance to different points in a ground truth; the closer to the ego, the more important. The weighting function is then incorporated into the formulation of IoU, resulting in the proposed EC-IoU. As finding the areas in IoU involves Green's Theorem and the weighted version of which becomes hard to solve, we further present an approximation scheme for computing EC-IoU via the Mean Value Theorem. We validate the approximation against the original curve computed by Monte Carlo integration and show it has the same time complexity as IoU.

EC-IoU can be used easily in common object detector evaluation pipelines. We incorporate it into two types of protocols, represented by the nuScenes~\cite{caesar2020nuscenes} and KITTI~\cite{geiger2012are} benchmarks\footnote{The utilization of the nuScenes and KITTI datasets in this paper is for knowledge dissemination and scientific publication and is not for commercial use.}. Our evaluation of several state-of-the-art models hosted on the MMDetection3D platform~\cite{mmdet3d2020} reveals that while achieving good accuracy (e.g., in terms of IoU), some of them may exhibit safety concerns. Furthermore, we utilize EC-IoU to train and fine-tune a model for more explicit safety awareness. While doing so, the optimization results show that our model also achieves higher mean Average Precision (mAP), the most widely used accuracy-based performance indicator, compared to an IoU-based-variant. Altogether, our work puts forth a novel instantiation of concretizing safety principles in the development of learning-based object detectors.

\section{Related Work}
\label{sec:related_work}
Object detection has been a long-standing research field, where readers may refer to recent surveys for a comprehensive overview of common sensors, detection algorithms, and available benchmarks~\cite{wu2020recent,ma2023object,mao2023object}. We focus here on the evolution of the localization/regression branch. 

Early approaches for object detection apply handcrafted filters to match patterns and locate objects in the input~\cite{lowe1999object,viola2001rapid}. Since the ``deep learning era," studies have considered various loss functions to regress model predictions towards ground truths via gradient descent. The earlier ones, such as the RCNN family~\cite{ren2015faster} and SSD~\cite{liu2016ssd}, are based on $L_1$ or $L_2$ norms, computing the numerical discrepancies in the object representing parameters (e.g., position, dimension, and orientation). Such approaches, however, result in normalization issues and ignore the spatial characteristics of the task. Therefore, later studies proposed IoU~\cite{yu2016unitbox} and Generalized-IoU~\cite{rezatofighi2018generalized} as the metrics and loss functions directly. More recently, they are augmented with regularizing terms and extended into Distance-IoU~\cite{zheng2019distanceiou} and Efficient-IoU~\cite{zhang2022focal} loss functions, achieving higher accuracy and faster convergence during training. Lastly, similar to Focal Loss~\cite{lin2017focal} for object classification, Focal-Efficient-IoU~\cite{zhang2022focal} and Alpha-IoU~\cite{he2022alphaiou} are proposed to find the effective samples and further improve the learning outcome.

The aforementioned results mainly fall into the scope of making the training more efficient and the object detector more accurate. Holding slightly different perspectives, some other studies have formulated considerations beyond accuracy. For instance, Waymo proposed the LET-3D-AP (Longitudinal-Error-Tolerating 3D-Average-Precision) metric, a relaxed evaluation protocol for camera-based object detectors considering their tendency to have larger longitudinal errors for objects at further distances from the ego vehicle~\cite{hung2022let3dap}. Focusing more on safety, the company suggested SDE (Support Distance Error), calculating absolute localization errors with reference to the ego vehicle's driving direction~\cite{deng2021revisiting}. Nonetheless, SDE was tailored for lidar-based models. More generalizable safety-oriented metrics can be found in~\cite{volk2020comprehensive,cheng2020safety,lyssenko2022towards}, in which model processing time and individual object importance (based on distances or expected time-to-collision) are taken into account. Still, in these works, when assessing predictions at the low level, only the ordinary IoU measure is used. Our work thereby complements these proposals with a more fine-grained measure. 

Finally, in terms of model compensation or refinement, the literature offers several proposals, e.g., robust learning against feature-level perturbations~\cite{cheng2020safety} and 2D bounding box post-processing based on worst-case analysis~\cite{schuster2022unaligned} or statistical approaches such as conformal prediction~\cite{degrancey2022object}. As we shall see, our work distinctively addresses 3D object detection using the bird's-eye-view (BEV) representation, which composes a more direct link with the downstream planning functions. We also amend the state-of-the-art loss functions (as mentioned above) with the proposed EC-IoU measure to achieve higher safety potential.

\section{Preliminaries}
\label{sec:preliminaries}

In this work, we follow the common practice of representing ground truths and predictions with 2D oriented bounding boxes in driving contexts~\cite{liu2022bevfusion}. Assuming the ego at the origin $O(0,0)$, a ground truth is normally annotated as a tuple $\hat{\mathbf{G}} \defeq (x_\mathbf{G}, y_\mathbf{G}, l_\mathbf{G}, w_\mathbf{G}, \theta_\mathbf{G})$, where $x_\mathbf{G}$ and $ y_\mathbf{G}$ are the center coordinates, $l_\mathbf{G}$ is the length (parallel with the $x$-axis when the orientation of the box is $0$), $w_\mathbf{G}$ is the width, and $\theta_\mathbf{G}$ the orientation. Similarly, a prediction, made by the object detector, can be represented by $\hat{\mathbf{P}} \defeq (x_\mathbf{P}, y_\mathbf{P}, l_\mathbf{P}, w_\mathbf{P}, \theta_\mathbf{P})$.

For evaluating model predictions, we use $\mathbf{P} \subset \mathbb{R}^2$ and $\mathbf{G} \subset \mathbb{R}^2$ to denote the 2D polygons (or, essentially, oriented bounding boxes) of the prediction~$\hat{\mathbf{P}}$ and ground truth~$\hat{\mathbf{G}}$. Technically, $\mathbf{P}$ and $\mathbf{G}$ can be reconstructed from $\hat{\mathbf{P}}$ and $\hat{\mathbf{G}}$ through their corners. Then, the IoU measure is given as:
\begin{equation}
    \mathsf{IoU}(\mathbf{P}, \mathbf{G}) \defeq \frac{\mathsf{Area}(\mathbf{P} \cap \mathbf{G})}{\mathsf{Area}(\mathbf{G}) + \mathsf{Area}(\mathbf{P}) - \mathsf{Area}(\mathbf{P} \cap \mathbf{G}) } \, ,
\label{eq:iou}
\end{equation}
where $\mathbf{P} \cap \mathbf{G}$ is the intersection of $\mathbf{P}$ and $\mathbf{G}$, and 
\begin{equation}
    \mathsf{Area}(\mathbf{D}) = \iint_\mathbf{D} 1 \,dA \, ,
\label{eq:area}
\end{equation}
for a polygon $\mathbf{D} \subset \mathbb{R}^2$. In the case of an oriented bounding box, the area can be simply calculated with its dimensions given in the tuple representation, e.g., $\mathsf{Area}(\mathbf{P}) = l_\mathbf{P} \times w_\mathbf{P}$. As for general polygons (e.g., $\mathbf{P} \cap \mathbf{G}$), there exists the Shoelace Formula to calculate their areas with their vertices' coordinates, which can be obtained through modern algorithms and software packages such as Shapely~\cite{gilles2023shapely}.

To illustrate, in Fig.~\ref{fig:concept}, all the predictions (including the blue and red ones) have an IoU of roughly~$0.7$ with respect to their corresponding ground truths. Still, as introduced, IoU only characterizes object-centric relations between the predictions and the ground truths generically. We now propose the EC-IoU measure, which additionally considers the ego position to better reflect the tendency of collisions in case of imperfect predictions around the ground truths.

\section{The EC-IoU Measure}

Overall, our approach is to first define safety-critical points in a ground truth and then check how well a prediction covers these safety-critical points.

\subsection{Safety-critical weighting for a ground truth}

To characterize safety-criticality for a ground truth $\mathbf{G}$, we start with a distance-based weighting function:
\begin{equation}
\label{eq:weighting}
    \omega_\mathbf{G}(x, y) \defeq \left[ \frac{\rho(x_\mathbf{G}, y_\mathbf{G})}{\rho(x, y)} \right]^{\alpha},
\end{equation}
where $(x, y) \in \mathbf{G}$ is a point within the ground truth, $\rho(x, y)$ denotes the Euclidean distance from the point $(x,y)$ to the origin (i.e., the position of the ego vehicle), and $\alpha \geq 0$ is a tunable parameter for adapting the basic weighting factor.

Essentially, the function in Eq.~\eqref{eq:weighting} weighs each point of the ground truth according to its distance to the origin. The closer the point is to the origin, the higher its weight, symbolizing its safety-criticality. The weights of different points are normalized with the ground truth center's distance so that all ground truths' centers have an equal weight of $1$. Using Fig.~\ref{fig:concept} as an example, for the pedestrian $\mathbf{G_{ped}}$, the truck $\mathbf{G_{truck}}$ and the car $\mathbf{G_{car}}$, we have $\omega_\mathbf{G_{ped}}(x_\mathbf{G_{ped}}, y_\mathbf{G_{ped}}) = \omega_\mathbf{G_{truck}}(x_\mathbf{G_{truck}}, y_\mathbf{G_{truck}}) = \omega_\mathbf{G_{car}}(x_\mathbf{G_{car}}, y_\mathbf{G_{car}}) = 1$. In effect, the normalization aligns the weights for all ground truths, as our goal is to rank predictions around a specific ground truth, instead of ranking different ground truths. In addition, it brings an implicit property: As a ground truth gets farther from the origin, the weighting function has less influence on it. We elaborate on this in the following lemma.

\vspace{1mm}

\begin{lemma}
    Given a ground truth $\mathbf{G}$ represented by $(x_\mathbf{G}, y_\mathbf{G}, l_\mathbf{G}, w_\mathbf{G}, \theta_\mathbf{G})$ with $\overline{\omega_\mathbf{G}}$ and $\underline{\omega_\mathbf{G}}$ being its maximum and minimum weights, if $\rho(x_\mathbf{G}, y_\mathbf{G}) \rightarrow \infty$, then $\overline{\omega_\mathbf{G}} = \underline{\omega_\mathbf{G}} = 1$.
    \label{lem:influence}
\end{lemma}

\begin{proof}
    Let $(\overline{x}, \overline{y}) \in \mathbf{G}$ be the point such that $\forall (x,y) \in \mathbf{G}:   \omega_\mathbf{G}(\overline{x}, \overline{y})  \geq  \omega_\mathbf{G}(x, y)$, i.e., $\omega_\mathbf{G}(\overline{x}, \overline{y}) = \overline{\omega_\mathbf{G}}$. 
    Now, with the center $(x_\mathbf{G}, y_\mathbf{G})$, the point $(\overline{x},\overline{y})$, and the origin $(0, 0)$, we can write the triangle inequality as:
    \begin{equation}
    \begin{split}
        \rho(x_\mathbf{G}, y_\mathbf{G}) & \leq \rho(\overline{x}, \overline{y}) + 
        \rho(x_\mathbf{G} - \overline{x}, y_\mathbf{G} - \overline{y} ) \\
        & \leq \rho(\overline{x}, \overline{y}) + \rho( {l_\mathbf{G}}/{2}, {w_\mathbf{G}}/{2} ),
    \end{split}
    \label{eq:tri_ineq}
    \end{equation}
    where $\rho(x_\mathbf{G} - \overline{x}, y_\mathbf{G} - \overline{y})$ equals the distance between the center $(x_\mathbf{G}, y_\mathbf{G})$ and the point $(\overline{x},\overline{y})$, which is guaranteed to be smaller or equal to the distance between the center and the corner of the rectangle, i.e., $\rho( {l_\mathbf{G}}/{2}, {w_\mathbf{G}}/{2})$.
 
    Then, considering the ground truth has constant dimensions $l_\mathbf{G}$ and $w_\mathbf{G}$, we can rewrite Eq.~\eqref{eq:tri_ineq} into:
    \begin{equation}
        \rho(x_\mathbf{G}, y_\mathbf{G}) = \rho(\overline{x}, \overline{y}) + c,
    \end{equation}
    where $c$ is a constant with $0 < c \leq \rho({l_\mathbf{G}}/{2}, {w_\mathbf{G}}/{2})$. Finally, taking the definition in Eq.~\eqref{eq:weighting}, we obtain:
    \begin{equation}
    \begin{split}
        \lim_{\rho(x_\mathbf{G}, y_\mathbf{G}) \to \infty} \overline{\omega_\mathbf{G}} 
        & = \lim_{\rho(x_\mathbf{G}, y_\mathbf{G}) \to \infty} \omega_\mathbf{G}(\overline{x}, \overline{y}) \\
        & = \lim_{\rho(x_\mathbf{G}, y_\mathbf{G}) \to \infty} \left[ \frac{\rho(x_\mathbf{G}, y_\mathbf{G})}{\rho(\overline{x}, \overline{y})} \right]^{\alpha} \\
        & = \lim_{\rho(x_\mathbf{G}, y_\mathbf{G}) \to \infty} \left[ \frac{\rho(x_\mathbf{G}, y_\mathbf{G})}{\rho(x_\mathbf{G}, y_\mathbf{G}) - c} \right]^{\alpha} \\
        & = \lim_{\rho(x_\mathbf{G}, y_\mathbf{G}) \to \infty} \left[ \frac{1}{1 - {\frac{c}{\rho(x_\mathbf{G}, y_\mathbf{G})}}} \right]^{\alpha} \\
        & = \left[ \frac{1}{1-0} \right]^{\alpha} = 1^\alpha = 1.
    \end{split}
    \end{equation}
    Similarly, it can be shown that $\lim_{\rho(x_\mathbf{G}, y_\mathbf{G}) \to \infty} \underline{\omega_\mathbf{G}} = 1$.
\end{proof}

\vspace{2mm}

In other words, a ground truth near the ego will have a larger weight difference between its closest and farthest points, corresponding to the safety notion that nearer objects should be handled with more care.  

\subsection{Formulation of the EC-IoU Measure}

With the presented weighting function, we now define our Ego-Centric Intersection-over-Union (EC-IoU) measure, which shall give a score to a prediction~$\mathbf{P}$ based on the safety-criticality (i.e., importance) of its overlap with a ground truth~$\mathbf{G}$:

\vspace{-2mm}
\begin{equation}
\begin{split}
    & \mathsf{EC\text{-}IoU}(\mathbf{P}, \mathbf{G}) \\
    & \defeq \frac{\mathsf{Weighted\text{-}Area}_\mathbf{G}(\mathbf{P} \cap \mathbf{G})}{\mathsf{Weighted\text{-}Area}_\mathbf{G}(\mathbf{G}) + \mathsf{Area}(\mathbf{P}) - \mathsf{Area}(\mathbf{P} \cap \mathbf{G}) }, 
\end{split}
\label{eq:ec_iou}
\end{equation}
where 
\vspace{-2mm}
\begin{align}
    & \mathsf{Weighted\text{-}Area}_\mathbf{G}(\mathbf{D}) = \iint_\mathbf{D} \omega_\mathbf{G}(x, y) \, dA \, ,
    \label{eq:weighted_area}
\end{align}
for a polygon $\mathbf{D} \subseteq \mathbf{G}$. 
To explain, $\mathsf{Weighted\text{-}Area}_\mathbf{G}(\mathbf{D})$ is the importance-weighted area of a polygon (within the ground truth) defined as the sum of the point weights therein. Then, to reflect safety-criticality in the relation between $\mathbf{P}$ and $\mathbf{G}$, we compute such importance-weighted area for their intersection and divide it by the weighted area of the ground truth itself~(which is the maximum that the intersection can achieve). Lastly, similar to the ordinary IoU measure~\eqref{eq:iou}, we keep the term $\mathsf{Area}(\mathbf{P}) - \mathsf{Area}(\mathbf{P} \cap \mathbf{G})$ in the denominator to avoid an excessively large prediction.

With the proposed formulation, EC-IoU has two useful properties. Firstly, it is contained within the range~$[0, 1]$. Secondly, it maximizes at~$1$ if and only if a prediction is perfectly aligned with a ground truth. The following lemmas prove them, respectively.

\vspace{1mm}
\begin{lemma}
Given a ground truth $\mathbf{G}$ and an arbitrary prediction $\mathbf{P}$, $0 \leq \mathsf{EC\text{-}IoU}(\mathbf{P}, \mathbf{G}) \leq 1$.
\label{lem:bounds}
\end{lemma}

\vspace{1mm}
\begin{proof}
Since $\mathbf{P} \cap \mathbf{G} \subseteq \mathbf{P}$,
\begin{equation}
    \mathsf{Area}(\mathbf{P}) - \mathsf{Area}(\mathbf{P} \cap \mathbf{G}) \geq 0 \, .
\label{eq:p_ineq}
\end{equation}
With $\mathsf{Weighted\text{-}Area}_\mathbf{G}(\mathbf{P} \cap \mathbf{G}) \geq 0$ and \\$\mathsf{Weighted\text{-}Area}_\mathbf{G}(\mathbf{G}) > 0$, the numerator in~\eqref{eq:ec_iou} is always larger than or equal to $0$, and the denominator is always larger than $0$. Hence, $\mathsf{EC\text{-}IoU}(\mathbf{P}, \mathbf{G}) \geq 0$. 

Additionally, since $\mathbf{P} \cap \mathbf{G} \subseteq \mathbf{G}$, 
\begin{equation}
    \mathsf{Weighted\text{-}Area}_\mathbf{G}(\mathbf{G}) - \mathsf{Weighted\text{-}Area}_\mathbf{G}(\mathbf{P} \cap \mathbf{G}) \geq 0 \, .
\label{eq:weighted_g_ineq}
\end{equation}
We combine it with~\eqref{eq:p_ineq}: 
\begin{equation}
\begin{split}
\mathsf{Weighted\text{-}Area}_\mathbf{G} & (\mathbf{P} \cap \mathbf{G}) - \mathsf{Weighted\text{-}Area}_\mathbf{G}(\mathbf{G}) \\ 
& \leq 0 \leq \mathsf{Area}(\mathbf{P}) - \mathsf{Area}(\mathbf{P} \cap \mathbf{G}) \, .
\end{split}
\end{equation}
Then, with a rewriting:
\begin{equation}
\begin{split}
& \mathsf{Weighted\text{-}Area}_\mathbf{G} (\mathbf{P} \cap \mathbf{G}) \\ 
& \leq \mathsf{Weighted\text{-}Area}_\mathbf{G}(\mathbf{G}) + \mathsf{Area}(\mathbf{P}) - \mathsf{Area}(\mathbf{P} \cap \mathbf{G}) \, .
\end{split}
\end{equation}
Finally, dividing both sides by the right-hand side and considering it is always larger than $0$, we have $\mathsf{EC\text{-}IoU}(\mathbf{P}, \mathbf{G}) \leq 1$. 
\end{proof}

\begin{lemma}
Given a ground truth~$\mathbf{G}$ and a prediction~$\mathbf{P}$, 
$\mathsf{EC\text{-}IoU}(\mathbf{P}, \mathbf{G}) = 1 \iff \mathbf{P}=\mathbf{G}$.
\label{lem:p_eq_g}
\end{lemma}

\begin{proof}
We first prove the left implication ($\Leftarrow$). With $\mathbf{P} = \mathbf{G}$, we have $\mathsf{Area}(\mathbf{P}) - \mathsf{Area}(\mathbf{P} \cap \mathbf{G}) = 0$ and  $\mathsf{Weighted\text{-}Area}_\mathbf{G}(\mathbf{P} \cap \mathbf{G}) = \mathsf{Weighted\text{-}Area}_\mathbf{G}(\mathbf{G})$. Therefore, 
\begin{equation}
    \mathsf{EC\text{-}IoU}(\mathbf{P}, \mathbf{G}) = \frac{\mathsf{Weighted\text{-}Area}_\mathbf{G}(\mathbf{G})}{\mathsf{Weighted\text{-}Area}_\mathbf{G}(\mathbf{G})} = 1 \, .
\end{equation}

Now, we prove the right implication ($\Rightarrow$) by contradiction. Given $\mathsf{EC\text{-}IoU}(\mathbf{P}, \mathbf{G}) = 1$, we have
\begin{equation}
\begin{split}
    \mathsf{Weighted\text{-}Area}_\mathbf{G}(\mathbf{P} \cap \mathbf{G}) - \mathsf{Weighted\text{-}Area}_\mathbf{G}(\mathbf{G}) \\ 
    + \mathsf{Area}(\mathbf{P} \cap \mathbf{G}) - \mathsf{Area}(\mathbf{P}) = 0.
\end{split}
\end{equation}
With Eq.~\eqref{eq:p_ineq} and Eq.~\eqref{eq:weighted_g_ineq}, we arrive at
\begin{align}
    \mathsf{Weighted\text{-}Area}_\mathbf{G} (\mathbf{P} \cap \mathbf{G}) & = \mathsf{Weighted\text{-}Area}_\mathbf{G}(\mathbf{G}) \label{eq:eq_wa}, \\
    \mathsf{Area}(\mathbf{P}) & = \mathsf{Area}(\mathbf{P} \cap \mathbf{G}) \label{eq:eq_a} .
\end{align}
Then, assuming $\mathbf{P} \neq \mathbf{G}$, we see that Eq.~\eqref{eq:eq_wa} holds only when $\mathbf{G}$ is contained by $\mathbf{P}$ (i.e., $\mathbf{G} \subset \mathbf{P}$), and Eq.~\eqref{eq:eq_a} holds only when $\mathbf{P}$ is contained by $\mathbf{G}$ (i.e., $\mathbf{P} \subset \mathbf{G}$). Clearly, these two cases never take place at the same time, leading to $\mathsf{EC\text{-}IoU}(\mathbf{P}, \mathbf{G}) < 1$, which contradicts the given premise. Hence, if $\mathsf{EC\text{-}IoU}(\mathbf{P}, \mathbf{G}) = 1$, then $\mathbf{P} = \mathbf{G}$.
\end{proof}

\vspace{2mm}

The two properties shown above make EC-IoU an eligible measure that quantifies the quality of predictions within the range $[0, 1]$ and secures the highest score for the best prediction. In the remaining part of the section, we show that, as opposed to IoU, EC-IoU characterizes our goal to rank a prediction higher if it overlaps with a more important portion of the ground truth from the ego's perspective.

\vspace{1mm}

\begin{lemma}
    Given a ground truth $\mathbf{G}$ and two predictions $\mathbf{P}_1$ and $\mathbf{P}_2$ having the same size, i.e., $\mathsf{Area}(\mathbf{P}_1) = \mathsf{Area}(\mathbf{P}_2)$, and the same IoU with the ground truth, i.e., $\mathsf{IoU}(\mathbf{P}_1, \mathbf{G}) = \mathsf{IoU}(\mathbf{P}_2, \mathbf{G})$, if $\mathsf{EC\text{-}IoU}(\mathbf{P}_1, \mathbf{G}) > \mathsf{EC\text{-}IoU}(\mathbf{P}_2, \mathbf{G})$, then $\mathsf{Weighted\text{-}Area}_\mathbf{G}(\mathbf{P}_1 \cap \mathbf{G}) > \mathsf{Weighted\text{-}Area}_\mathbf{G}(\mathbf{P}_2 \cap \mathbf{G})$ 
    and $\iint_\mathbf{\mathbf{P}_1 \cap \mathbf{G}} \rho(x, y) \,dA < \iint_\mathbf{\mathbf{P}_2 \cap \mathbf{G}} \rho(x, y) \,dA$.
\label{lem:comparison}
\end{lemma}

\vspace{2mm}
\begin{proof}
    Since $\mathsf{Area}(\mathbf{P}_1) = \mathsf{Area}(\mathbf{P}_2)$ and $\mathsf{IoU}(\mathbf{P}_1, \mathbf{G}) = \mathsf{IoU}(\mathbf{P}_2, \mathbf{G})$, based on the definition of IoU in Eq.~\eqref{eq:iou}, one can derive that $\mathsf{Area}(\mathbf{P}_1 \cap \mathbf{G}) = \mathsf{Area}(\mathbf{P}_2 \cap \mathbf{G})$. Therefore, $\mathsf{Weighted\text{-}Area}_\mathbf{G}(\mathbf{G}) + \mathsf{Area}(\mathbf{P}_1) - \mathsf{Area}(\mathbf{P}_1 \cap \mathbf{G}) = \mathsf{Weighted\text{-}Area}_\mathbf{G}(\mathbf{G}) + \mathsf{Area}(\mathbf{P}_2) - \mathsf{Area}(\mathbf{P}_2 \cap \mathbf{G})$, i.e., the denominator in Eq.~\eqref{eq:ec_iou} is the same for $\mathbf{P}_1$ and $\mathbf{P}_2$. Then, since $\mathsf{EC\text{-}IoU}(\mathbf{P}_1, \mathbf{G}) > \mathsf{EC\text{-}IoU}(\mathbf{P}_2, \mathbf{G})$, it follows that $\mathsf{Weighted\text{-}Area}_\mathbf{G}(\mathbf{P}_1 \cap \mathbf{G}) > \mathsf{Weighted\text{-}Area}_\mathbf{G}(\mathbf{P}_2 \cap \mathbf{G})$ at the numerator. 

    Then, using the definition of $\mathsf{Weighted\text{-}Area}_\mathbf{G}(\cdot)$ and $\omega_\mathbf{G}(\cdot, \cdot)$, we obtain: 
    \begin{equation}
        \iint_\mathbf{\mathbf{P}_1 \cap \mathbf{G}} \left[ \frac{\rho(x_\mathbf{G}, y_\mathbf{G})}{\rho(x, y)} \right]^{\alpha} \,dA  > \iint_\mathbf{\mathbf{P}_2 \cap \mathbf{G}}\left[ \frac{\rho(x_\mathbf{G}, y_\mathbf{G})}{\rho(x, y)} \right]^{\alpha} \,dA.
    \end{equation}
    Eliminating the positive constant $\rho(x_\mathbf{G}, y_\mathbf{G}) ^ \alpha$ and considering $\alpha \geq 0$, we reach $\iint_\mathbf{\mathbf{P}_1 \cap \mathbf{G}} \rho(x, y) \,dA < \iint_\mathbf{\mathbf{P}_2 \cap \mathbf{G}} \rho(x, y) \,dA$.
\end{proof}

\begin{figure}[t]
    \centering
    \includegraphics[width=0.9\linewidth]{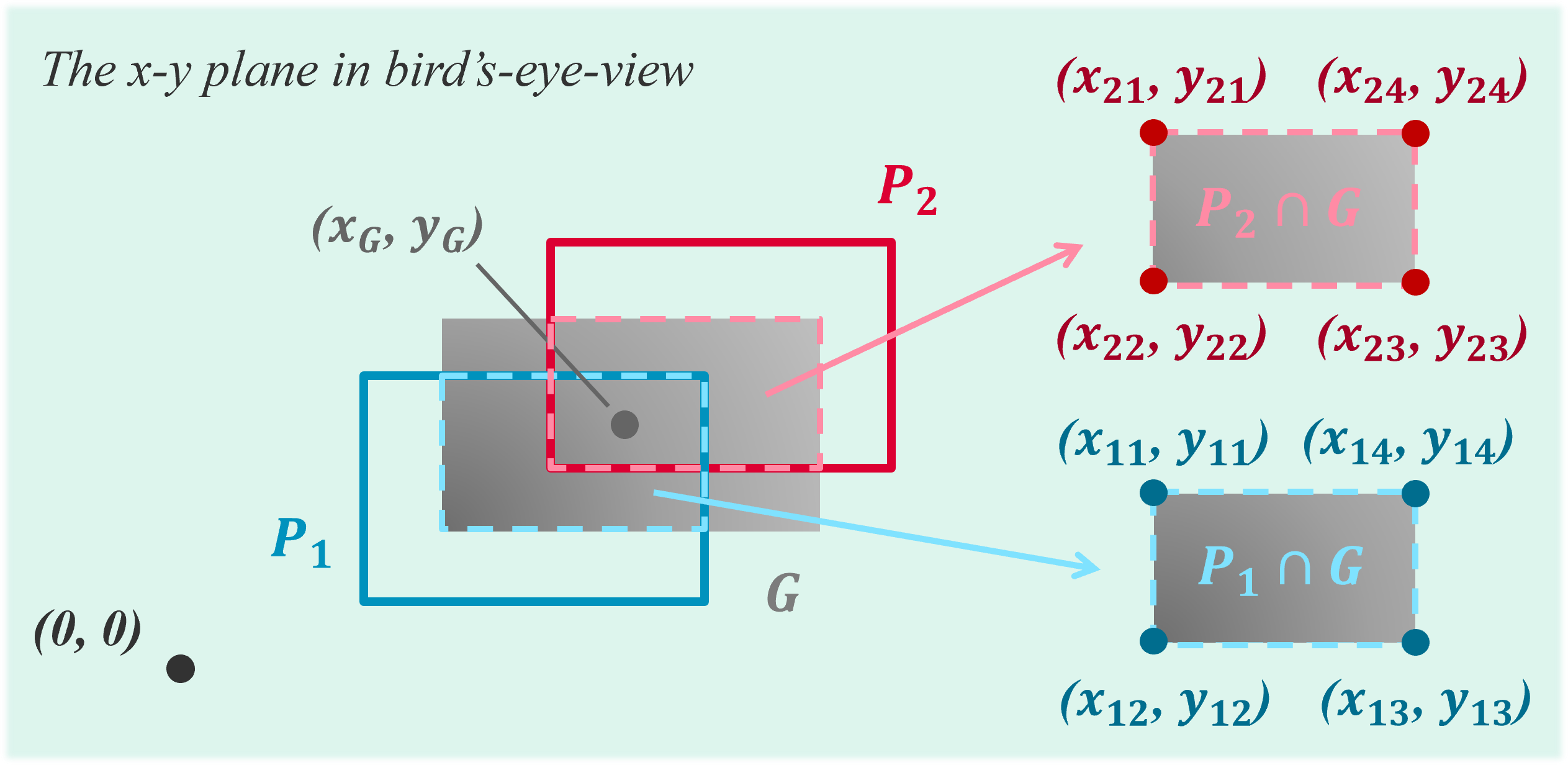}
    \caption{An example showing a prediction $\mathbf{P}_1$ will be favored by EC-IoU over another prediction $\mathbf{P}_2$ thanks to its better coverage on the safety-critical portion of the ground truth $\mathbf{G}$. As depicted by the gradient effect in $\mathbf{G}$, safety criticality is defined based on point distances to the origin; the darker, the more critical.}
    \label{fig:example}
    \vspace{-5mm}
\end{figure}

\vspace{2mm}

In layman's terms, compared to prediction $\mathbf{P}_2$, prediction $\mathbf{P}_1$ with a higher EC-IoU score collects more of the important points in the ground truth. Fig.~\ref{fig:example} provides an illustration, in which two predictions $\mathbf{P}_1$ and $\mathbf{P}_2$ attempt to match the ground truth and reach the same IoU. However, by considering the relative position of the ground truth with respect to the ego agent and, thereby, the proposed weighting mechanism $\omega_\mathbf{G}(x, y)$, we enable EC-IoU to favor the blue box (prediction $\mathbf{P}_1$) over the red one (prediction $\mathbf{P}_2$). 

One remaining challenge, however, is the computation of the weighted area of a polygon defined in Eq.~\eqref{eq:weighted_area}. Unlike the area of an unweighted polygon in Eq.~\eqref{eq:area}, which can be obtained by the closed-form Shoelace Formula via Green's Theorem, the weighted extension is hard to compute due to the variable weights and area boundaries (i.e., integral with variable limits). Hence, in the following, we introduce a method to efficiently approximate the weighted area of a polygon and, accordingly, the EC-IoU measure.

\subsection{Computing Weighted Areas and EC-IoU}
\label{subsec:computation}

The weighted area in Eq.~\eqref{eq:weighted_area} is essentially an integral of the weighting function over a variable region, which is hard to derive into a closed-form expression. Generally, one can employ numerical methods such as \textit{Monte Carlo integration} to compute such integrals. Our approach, inspired by the Mean Value Theorem, is to decompose the weighting and the calculation of the region's area. Essentially, the Mean Value Theorem provides that for a convex polygon~$\mathbf{D}$, $\exists \, (x_m, y_m) \in \mathbf{D}$: 
\begin{align}
\begin{split}
    & \mathsf{Weighted\text{-}Area}_\mathbf{G}(\mathbf{D}) = \iint_\mathbf{D} \omega_\mathbf{G}(x, y) \,dA \\
    & = \omega_\mathbf{G}(x_m, y_m) \cdot \iint_\mathbf{D} 1 \,dA \\
    & = \omega_\mathbf{G}(x_m, y_m) \cdot \mathsf{Area}(\mathbf{D}).
    \label{eq:mean_value}
\end{split}
\end{align}

Finding $(x_m, y_m)$ for weighting, unfortunately, is also highly non-trivial. Henceforth, we compute an approximation of such a mean weight $\omega_\mathbf{G}(x_m, y_m)$ by considering the nature of the convex polygon $\mathbf{D}$.
Effectively, we weigh the vertices of $\mathbf{D}$ and then find their central tendency, for which typical choices include the \emph{arithmetic mean} and \emph{geometric mean}. As the former is more suitable for cases where the underlying values do not differ drastically, we propose to use the latter to keep the result less sensitive to outlying vertices in the case of larger ground truths. Formally, it is defined as:
\begin{align}
    \omega_\mathbf{G,D} \defeq \left( \prod_{i=1}^{m} \omega_\mathbf{G}(x^\mathbf{D}_i, y^\mathbf{D}_i) \right)^{(1/m)},
    \label{eq:omega_geom}
\end{align}
for a polygon $\mathbf{D}$ with $m$ vertices. Taking again Fig.~\ref{fig:example} as an example, for the vertices of $\mathbf{P}_1 \cap \mathbf{G}$ and $\mathbf{P}_2 \cap \mathbf{G}$, their weights follow $\omega_\mathbf{G}(x_{1i}, y_{1i}) > \omega_\mathbf{G}(x_{2i}, y_{2i})$, where $i=1,2,3,4$. Therefore, applying the geometric mean (or the same interpolation scheme) to the vertex weights of $\mathbf{P}_1 \cap \mathbf{G}$ and $\mathbf{P}_2 \cap \mathbf{G}$ leads to $\omega_\mathbf{G,\mathbf{P}_1 \cap \mathbf{G}} > \omega_\mathbf{G,\mathbf{P}_2 \cap \mathbf{G}}$ and, hence, $\mathsf{Weighted\text{-}Area}_\mathbf{G}(\mathbf{P}_1 \cap \mathbf{G}) > \mathsf{Weighted\text{-}Area}_\mathbf{G}(\mathbf{P}_2 \cap \mathbf{G})$.

\begin{figure}[t]
\centering
    \includegraphics[width=0.65\linewidth]{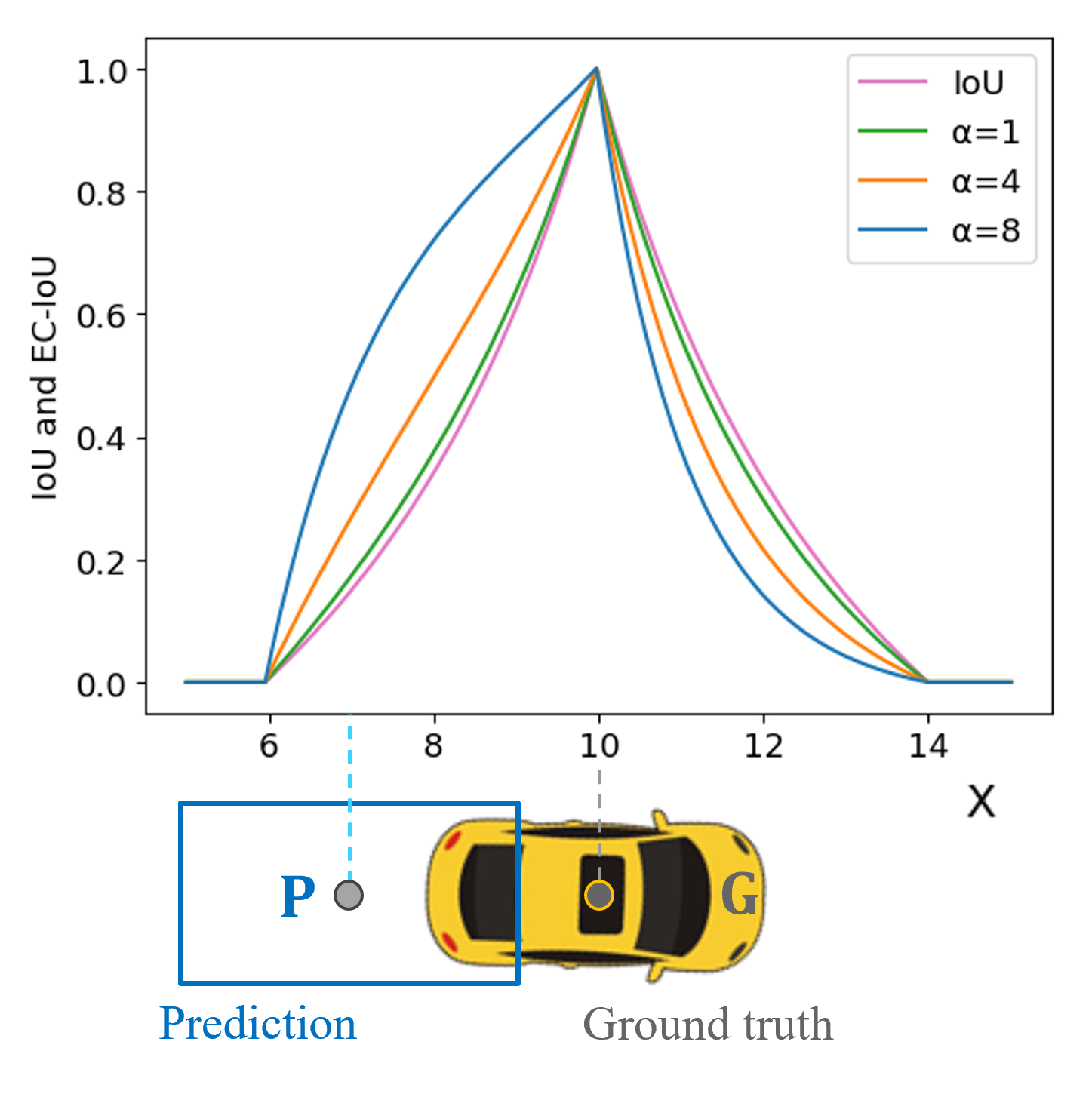}
        \vspace{-3mm}
    \caption{IoU and EC-IoU with various~$\alpha$, computed using the geometric mean, for predictions centered along the x-axis. We assume the ego vehicle is located at~$x=0$ and the ground truth~$\mathbf{G}$ at $x=10$. The blue box depicts a prediction~$\mathbf{P}$ centered at~$x=7$.}
    \label{fig:ec_iou_valuation}
    \vspace{-5mm}
\end{figure}

As such, we can compute EC-IoU using $\omega_{\mathbf{G},\mathbf{D}}$ as an approximation to $\omega_\mathbf{G}(x_m, y_m)$. Fig.~\ref{fig:ec_iou_valuation} illustrates a simulated result where we create a ground truth vehicle $\mathbf{G} = (10, 0, 4, 2, 0)$ and shift a prediction of the same dimension from $(5, 0)$ to $(15, 0)$, assuming the ego at $(0, 0)$. As shown, EC-IoU curves are higher than IoU for $x_\mathbf{P} \in [6, 10)$ and lower for $x_\mathbf{P} \in (10, 14]$. This indicates that, when compared to IoU, EC-IoU gives higher scores for predictions closer to the ego and discourages farther ones. The difference increases with larger values of $\alpha$.

While the EC-IoU in Fig.~\ref{fig:ec_iou_valuation} is computed using the geometric mean, we take the extreme $\alpha=8$ case and demonstrate in Fig.~\ref{fig:eciou_with_different_methods} the result of applying the arithmetic mean as well as the original curve generated by Monte Carlo numerical integration. The result illustrates the approximations with the geometric and arithmetic mean lead to highly similar curves and the error of the geometric mean is overall smaller. Finally, we note that using approximation may occasionally lead to cases where the computed EC-IoU exceeds~$1$. We observe such cases when setting $\alpha>16$ in Eq.~\eqref{eq:weighting}, for which we simply clamp the EC-IoU value so that the computed result remains within~$[0, 1]$.

\begin{figure}[t]
    \centering
    \includegraphics[width=0.65\linewidth]{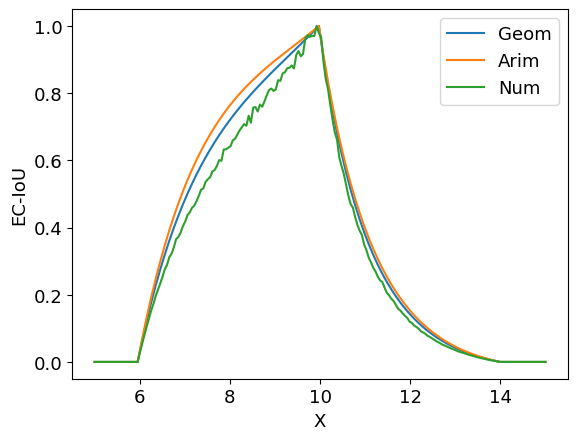}
    \caption{For EC-IoU with $\alpha=8$, under the same configuration as Fig.~\ref{fig:ec_iou_valuation}, we compare the curves produced by (1) geometric mean approximation ($\mathsf{Geom}$), (2) arithmetic mean approximation ($\mathsf{Arim}$), and (3) solving the original function via Monte Carlo numerical integration ($\mathsf{Num}$) with~$6000$ random samples for every prediction centered at~$x$. }
    \label{fig:eciou_with_different_methods}
\end{figure}

\subsection{Complexity and Usage}
\label{subsec:complexity}
In this final subsection, we analyze the time complexity of computing EC-IoU briefly and describe how it can be integrated into general model evaluation and optimization pipelines.

Given $n$ pairs of prediction $\mathbf{P}$ and ground truth $\mathbf{G}$, the ordinary IoU computation involves several steps. First, it starts with finding the vertices of $\mathbf{P} \cap \mathbf{Q}$, which is at the worst case $\mathcal{O}(n^3)$ for modern algorithms~\cite{berg2008computational}. Then, it sorts the vertices in order, taking another $\mathcal{O}(n \log n)$ for well-known techniques (e.g., merge sort). Lastly, as mentioned, Shoelace Formula can be applied to attain the area within $\mathcal{O}(n)$. Considering now the additional weighting step for EC-IoU, i.e., Eq.~\eqref{eq:omega_geom}, only another $\mathcal{O}(n)$ is needed because there are at most 8 intersecting vertices for each pair of prediction and ground truth and the weighting of a point takes constant time. As a result, the overall computation time of EC-IoU should be comparable to that of IoU.

EC-IoU can be employed in typical object detection evaluation protocols in two straightforward ways. First, similar to the nuScenes true-positive metrics~\cite{caesar2020nuscenes}, it can be used as a direct metric when assessing matched pairs of predictions and ground truths. Second, for protocols that match predictions and ground truths through IoU-based affinity, e.g., the KITTI benchmark~\cite{geiger2012are}, EC-IoU naturally provides an additional option and leads to weighted Average Precision (AP) metrics for safety characterization. Correspondingly, to improve such safety-oriented performance explicitly, EC-IoU can be integrated into common loss functions for model optimization. For example, considering the vanilla IoU~\cite{yu2016unitbox} and the more advanced DIoU~\cite{zheng2019distanceiou} and EIoU~\cite{zhang2022focal}, 
\begin{equation}\label{eq:iou_loss}
\begin{split}
    & L_{\{\textsf{-}/\textsf{D}/\textsf{E}\}\textsf{IoU}} \defeq 1 - \mathsf{IoU} + R_{\{\textsf{-}/\textsf{D}/\textsf{E}\}\textsf{IoU}}, \\ 
\end{split}
\end{equation}
we adapt them with EC-IoU:
\begin{equation}\label{eq:ec_iou_loss}
L_{\textsf{EC}\text{-}\{\textsf{-}/\textsf{D}/\textsf{E}\}\textsf{IoU}} \defeq 1 - \mathsf{EC\text{-}IoU} + R_{\{\textsf{-}/\textsf{D}/\textsf{E}\}\textsf{IoU}}, 
\end{equation}
where $\{\textsf{-}/\textsf{D}/\textsf{E}\}\textsf{IoU}$ denote the three cases IoU, DIoU, and EIoU, and $R_{\{\textsf{-}/\textsf{D}/\textsf{E}\}\textsf{IoU}}$ their respective regularization terms. In Eq.~\eqref{eq:iou_loss} and Eq.~\eqref{eq:ec_iou_loss}, the notations for the prediction~$\mathbf{P}$ and the ground truth~$\mathbf{G}$ are omitted for simplicity, and we take $\alpha=1$ in the EC-IoU-variants.

Finally, we note that for benchmarks supporting annotations with 3D bounding boxes (i.e., having the~$z$ coordinate and height~$h$ along the gravity axis in addition to the 2D representation tuple given in Sec.~\ref{sec:preliminaries}), EC-IoU can also be extended into 3D cases. Concretely, since the weighting mechanism does not concern the gravity axis, we simply multiply (weighted) areas with their heights to get volumes and compute the 3D measure. In consequence, we can address both BEV and 3D scenarios, as shown in the following experiments.

\section{Experimental Results and Discussions}
\label{sec:exp}

This section presents and discusses our experimental results, with both synthetic and real-world datasets.

\subsection{Benchmarking EC-IoU-Based Loss Functions}
\label{subsec:opt_sim}

\begin{wrapfigure}[12]{r}{0.48\linewidth}
\vspace{-4mm}
    \centering
    \includegraphics[width=\linewidth]{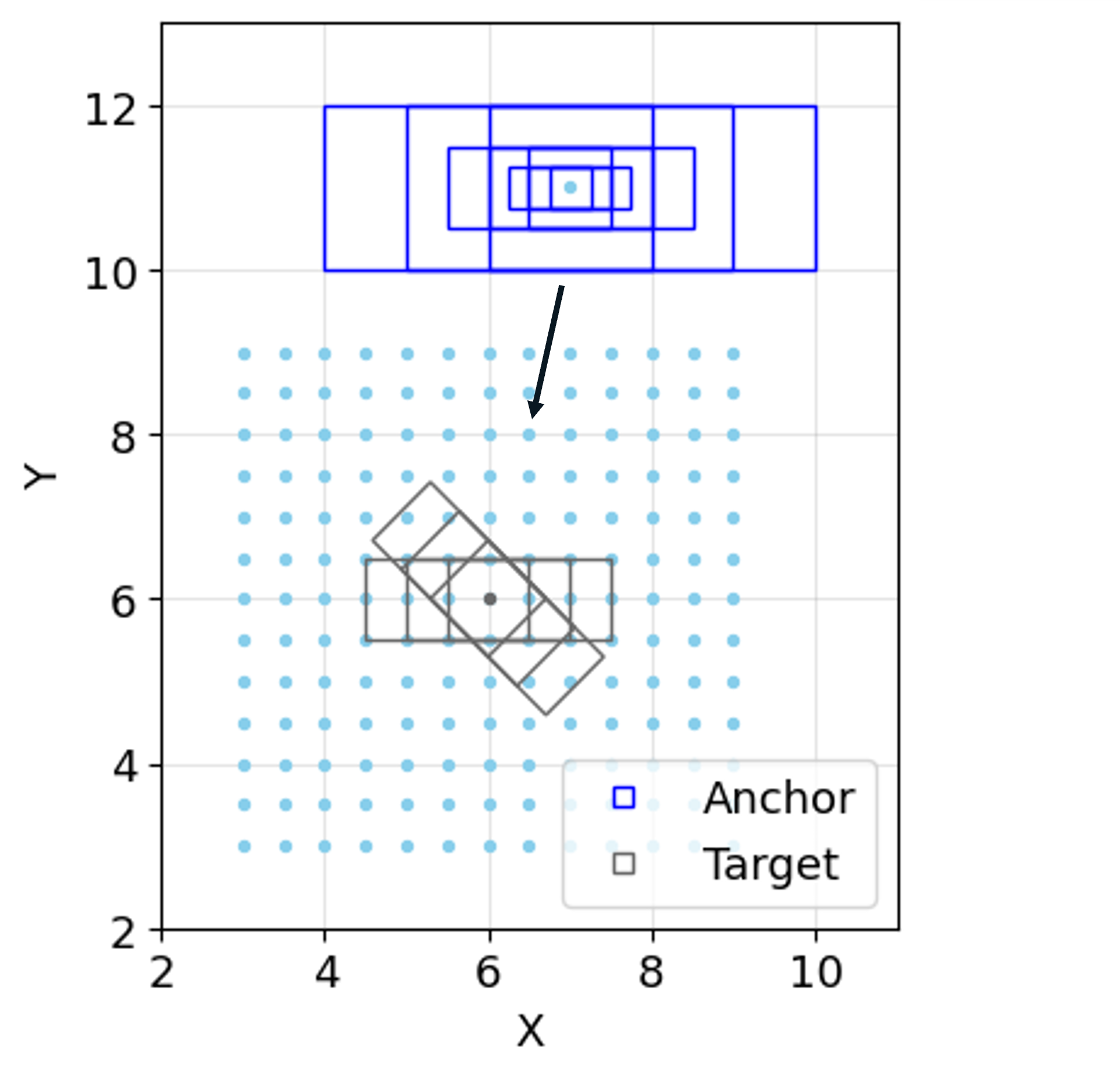}
    \caption{Simulation setup.}
    \label{fig:opt_sim_setup}
\end{wrapfigure}
We begin with a simulation experiment to inspect the EC-IoU-based loss functions and their counterparts. Similar to the prior art~\cite{zheng2019distanceiou,zhang2022focal}, we create anchors and targets on the $x$-$y$ plane and simulate a bounding box regression process. As shown in Fig.~\ref{fig:opt_sim_setup}, 6 targets are placed at $(6,6)$ with different dimensions ($(l, w)$ being $(1,1)$, $(2,1)$, or $(3,1)$) and orientations ($\theta$ being $0$ or $\pi/4$). Then, $13 \times 13$ points are sampled grid-wise within a 6-by-6 square centered at $(6, 6)$. Each point marks a set of $3 \times 3$ anchors, generated by the combination of three aspect ratios ($(1,1)$, $(2,1)$, and $(3,1)$) and three scales ($0.5$, $1$, and $2$). With the setup, we optimize the anchors towards the targets one at a time via gradient descent with 180 iterations, resulting in $9*6*13*13=9626$ regression cases. Finally, the process is repeated for all six loss functions given in Eq.~\ref{eq:iou_loss} and Eq.~\ref{eq:ec_iou_loss}. For pseudocode of the optimization process, readers may refer to~\cite{zheng2019distanceiou,zhang2022focal}.

In Fig.~\ref{fig:optim_sim}, we plot the performance of all six loss functions in terms of IoU and EC-IoU scores (with $\alpha=4$) across the 180 iterations. Three key observations are extracted as follows: (1) In EC-IoU assessment, the proposed EC-IoU-based loss functions, especially $L_{\textsf{EC}\text{-}\textsf{DIoU}}$, retain higher scores than the baselines throughout the optimization process. We note that the basic $L_{\textsf{EC}\text{-}\textsf{IoU}}$ performs only marginally better than $L_{\textsf{IoU}}$ here and conduct an ablation of them in Sec.~\ref{subsec:detector_finetuning}. Still, while it is possible that IoU-based optimizer achieves with more iterations a good EC-IoU (when the anchors finally arrive at the targets), using EC-IoU-based ones boost the score faster. (2) In IoU assessment, the EC-IoU-based loss functions achieve comparable performance with the benchmarks, although they exhibit certain jumps during optimization. These jumps of EC-IoU-based loss functions in IoU scores are likely caused by the nature of the functions. Referring to Fig.~\ref{fig:ec_iou_valuation}, when a prediction is optimized by EC-IoU asymptotically to the target from the left, i.e., $x_\mathbf{P} \in [9,10]$, the IoU score increases drastically with its steep profile. From the other side, i.e., $x_\mathbf{P} \in [10,11]$, as the EC-IoU profile itself is steep and the gradient is large, the IoU score will be quickly lifted also. Hence, the jumps are not caused by learning instability but by the underlying functions. (3) With the EC-IoU curves reaching $0.6$\textasciitilde$0.8$ differently, Fig.~\ref{fig:optim_sim} confirms that it can better differentiate models that have saturating and similar IoU scores ($\approx 0.5$).

\begin{figure}
    \centering
    \includegraphics[width=\linewidth]{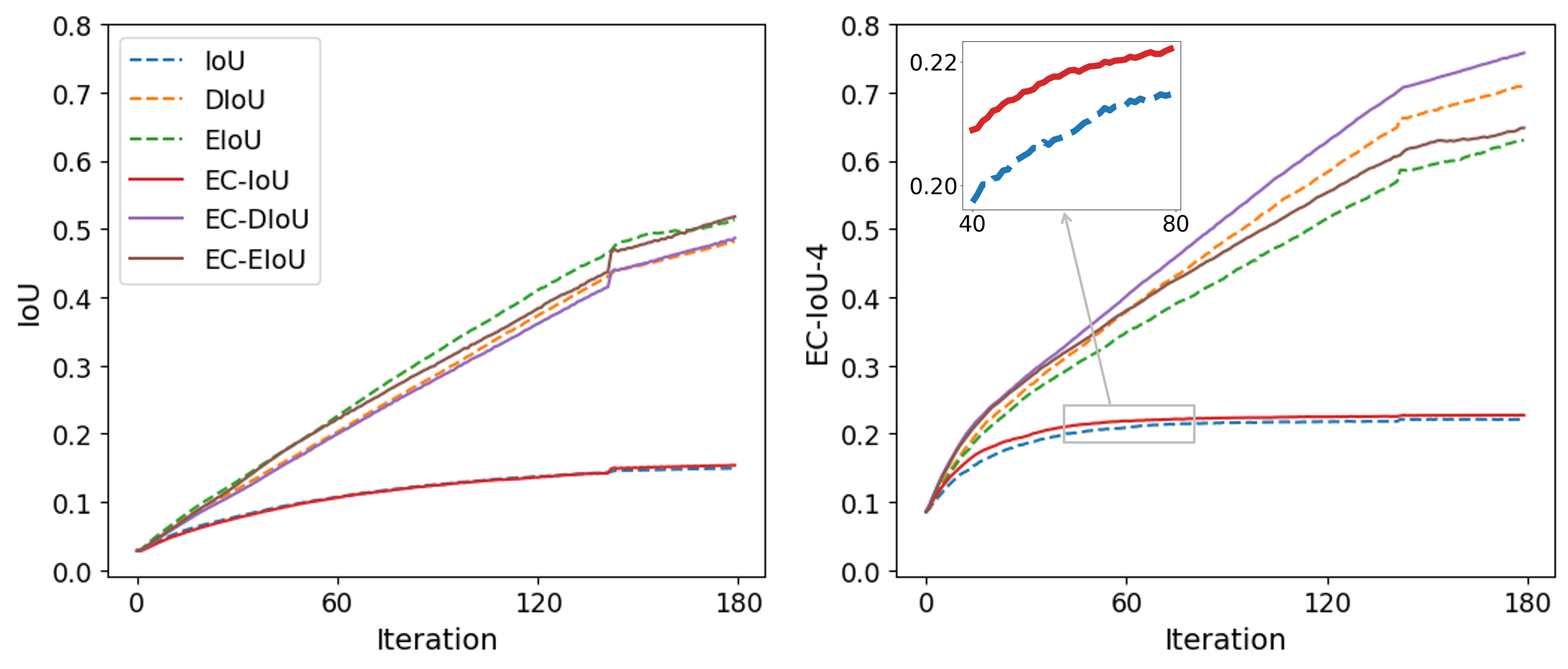}
    \caption{The optimization results of the proposed EC-IoU-based loss functions and their counterparts.} 
    \label{fig:optim_sim}
    \vspace{-5mm}
\end{figure}

\subsection{Real-World Object Detector Evaluation}
We now extend our experiment with real-world datasets, first evaluating popular object detectors using the nuScenes~\cite{caesar2020nuscenes} and KITTI~\cite{geiger2012are} benchmarks.

\begin{figure*}
    \centering
    \begin{minipage}[]{0.49\textwidth}
        \begin{minipage}[b]{0.57\linewidth}
            \includegraphics[width=\linewidth]{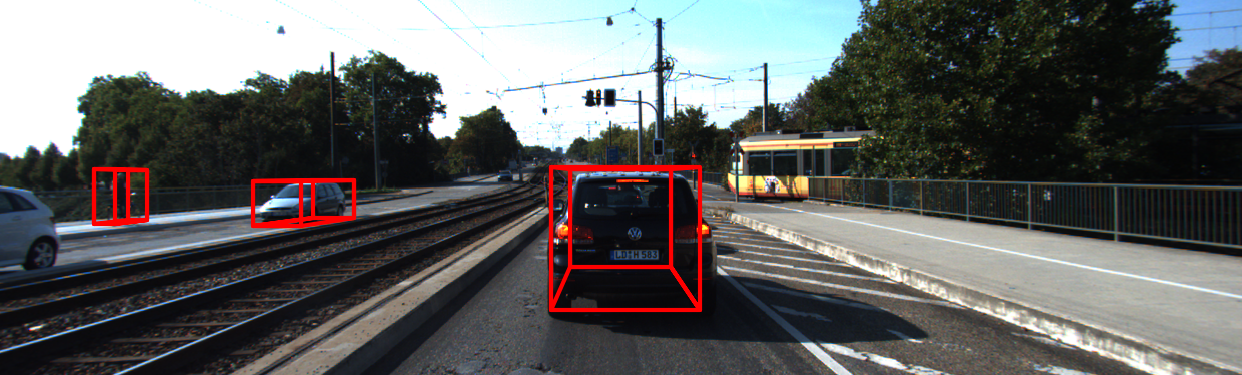} \vspace{3mm}
            \includegraphics[width=\linewidth]       {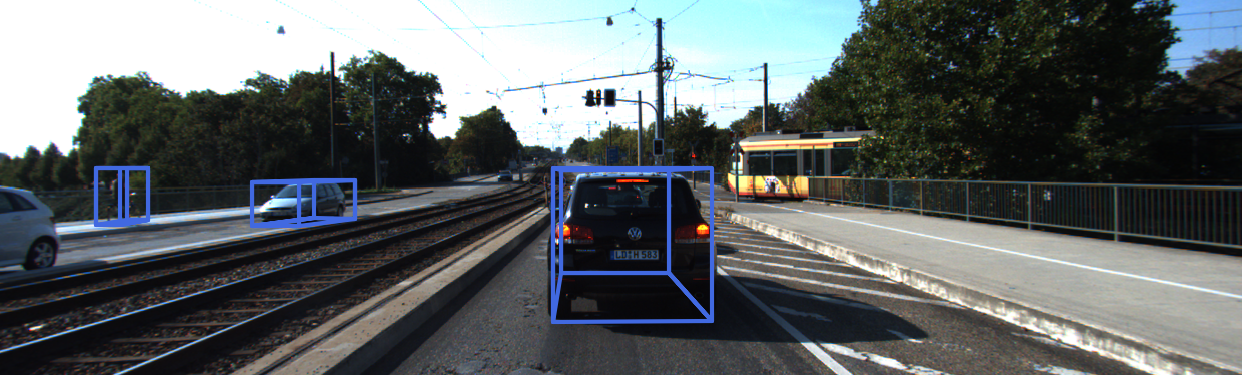}    
        \end{minipage}
        \includegraphics[width=0.41\linewidth]{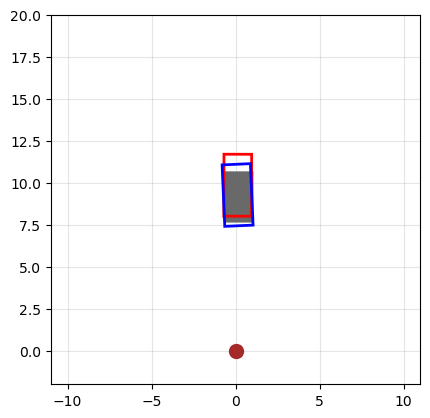}
    \end{minipage}
    \begin{minipage}[]{0.49\textwidth}
        \begin{minipage}[b]{0.57\linewidth}
            \includegraphics[width=\linewidth]{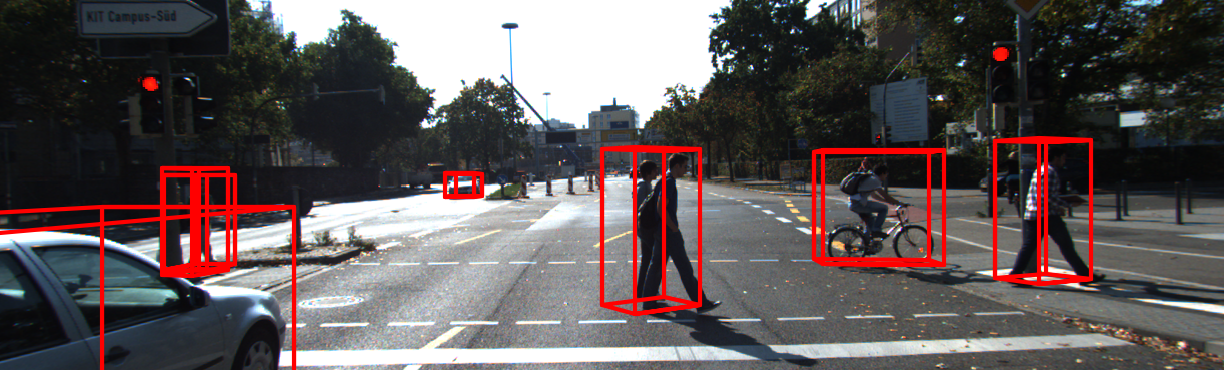} \vspace{3mm}
            \includegraphics[width=\linewidth]       {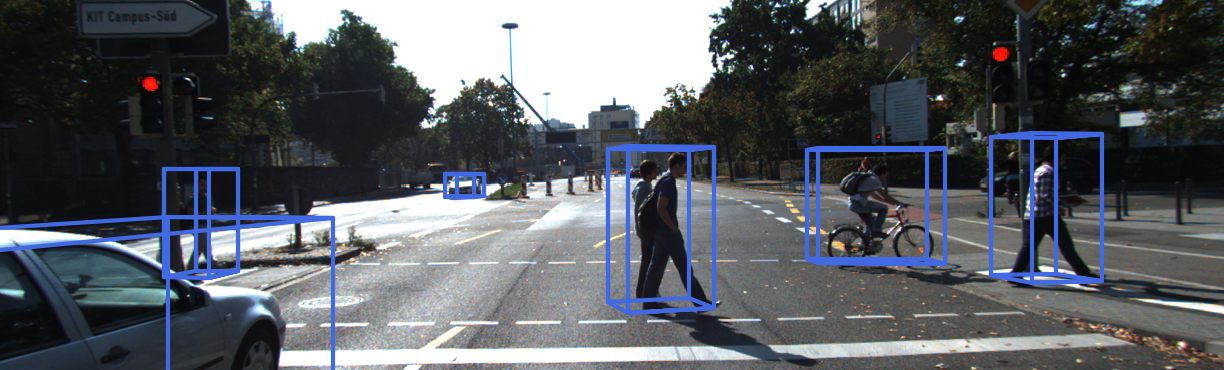}    
        \end{minipage}
        \includegraphics[width=0.41\linewidth]{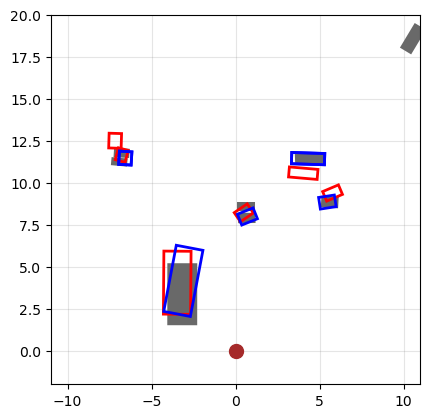}
    \end{minipage}
    \caption{Qualitative results of fine-tuning PDG~\cite{wang2021probabilistic} using the KITTI dataset~\cite{geiger2012are}. Each of the two scenes consists of \textit{(upper left)} red predictions by $L_\mathsf{IoU}$, \textit{(lower left)} blue predictions by $L_\mathsf{EC\text{-}IoU}$, and \textit{(right)} a BEV examination of the scene with both red and blue predictions as well as gray ground-truth boxes. The left scene demonstrates a case with more proper object coverage, whereas the right shows the possibility of further aligning the predictions towards the objects.}
    \label{fig:kitti_results}
\end{figure*}

\begin{table*}[]
\centering
\caption{Model evaluation results in the nuScenes dataset~\cite{caesar2020nuscenes} (Ped.=Pedestrian). The truck class has substantially lower EC-IoU scores from all models, indicating the difficulty of capturing it fully from the ego's point of view.}
\label{tab:nuscenes}
\begin{tabular}{l|c|c|cc|cc|cc}
\hline
\multirow{2}{*}{Model} & \multirow{2}{*}{Modality} & \multirow{2}{*}{NDS} & \multicolumn{2}{c|}{Car}           & \multicolumn{2}{c|}{Truck}                & \multicolumn{2}{c}{Ped.}           \\ \cline{4-9} 
                       &                       &                      & \multicolumn{1}{c|}{IoU}  & EC-IoU & \multicolumn{1}{c|}{IoU}  & EC-IoU        & \multicolumn{1}{c|}{IoU}  & EC-IoU \\ \hline
SSN~\cite{zhu2020ssn}                    & Lidar                  & 45.49                & \multicolumn{1}{c|}{0.74} & 0.73   & \multicolumn{1}{c|}{0.74} & \textit{0.61} & \multicolumn{1}{c|}{0.53} & 0.54   \\
CenterPoint~\cite{yin2021center}            & Lidar                  & 54.32                & \multicolumn{1}{c|}{0.76} & 0.74   & \multicolumn{1}{c|}{0.75} & \textit{0.67} & \multicolumn{1}{c|}{0.54} & 0.56   \\ \hline
FCOS3D~\cite{wang2021fcos3d}                 & Camera                  & 30.83                & \multicolumn{1}{c|}{0.60} & 0.61   & \multicolumn{1}{c|}{0.68} & \textit{0.53} & \multicolumn{1}{c|}{0.21} & 0.23   \\
PGD~\cite{wang2021probabilistic}                    & Camera                  & 31.49                & \multicolumn{1}{c|}{0.62} & 0.66   & \multicolumn{1}{c|}{0.71} & \textit{0.53} & \multicolumn{1}{c|}{0.23} & 0.25   \\ \hline
\end{tabular}
\end{table*}

\begin{table*}[]
\caption{Model evaluation and fine-tuning results in the KITTI dataset~\cite{geiger2012are} (Ped.=Pedestrian; Std.=Standard). Generally, the results from EC-IoU loss are not only safer in terms of EC-AP but also more accurate in terms of standard-AP.}
\label{tab:kitti}
\centering
\begin{tabular}{l|c|ccll|ccll|ccll}
\hline
\multirow{2}{*}{Model} & \multirow{2}{*}{Modality} & \multicolumn{4}{c|}{Ped. AP40@0.5}                      & \multicolumn{4}{c|}{Car AP40@0.5}                        & \multicolumn{4}{c}{mAP40}                              \\ \cline{3-14} 
                       &                       & \multicolumn{1}{c|}{Std.}  & \multicolumn{3}{c|}{EC} & \multicolumn{1}{c|}{Std.}   & \multicolumn{3}{c|}{EC} & \multicolumn{1}{c|}{Std.}  & \multicolumn{3}{c}{EC} \\ \hline
SMOKE~\cite{liu2020smoke}                  & Camera                  & \multicolumn{1}{c|}{3.57} & \multicolumn{3}{c|}{4.65}   & \multicolumn{1}{c|}{12.02} & \multicolumn{3}{c|}{16.83}  & \multicolumn{1}{c|}{5.35} & \multicolumn{3}{c}{7.48}   \\
PGD~\cite{wang2021probabilistic}                    & Camera                  & \multicolumn{1}{c|}{4.2}  & \multicolumn{3}{c|}{\textbf{5.59}}   & \multicolumn{1}{c|}{12.82} & \multicolumn{3}{c|}{\textit{14.48}}  & \multicolumn{1}{c|}{5.93} & \multicolumn{3}{c} {\textit{7.23}}   \\ \hdashline
+ $L_{\textsf{IoU}}$            & Camera                  & \multicolumn{1}{c|}{{3.72}} & \multicolumn{3}{c|}{{4.45}}   & \multicolumn{1}{c|}{13.58} & \multicolumn{3}{c|}{17.95 (+23.6\%)}  & \multicolumn{1}{c|}{7.21} & \multicolumn{3}{c}{9.36 (+29.5\%)}   \\
+ $L_{\textsf{EC}\text{-}\textsf{IoU}}$         & Camera                  & \multicolumn{1}{c|}{\textbf{4.5}}  & \multicolumn{3}{c|}{{5.26}}   & \multicolumn{1}{c|}{\textbf{14.63}} & \multicolumn{3}{c|}{\textbf{18.45} (+27.4\%)}  & \multicolumn{1}{c|}{\textbf{7.42}} & \multicolumn{3}{c}{\textbf{10.07} (+39.3\%)}  \\ \hline
\end{tabular}
    \vspace{-5mm}
\end{table*}

\subsubsection{nuScenes}

As described in Sec.~\ref{subsec:complexity}, having matched the predictions to the closest ground truths, nuScenes computes the NuScenes Detection Score (NDS) and a set of True-Positive (TP) metrics such as translational and rotational errors~\cite{caesar2020nuscenes}. As such, IoU and EC-IoU can naturally be two extra metrics to reflect the overall spatial relation between predictions and ground truths.

For implementation, we utilize the MMDetection3D platform~\cite{mmdet3d2020} and test top-performing models, including two lidar-based and two camera-based ones. Tab.~\ref{tab:nuscenes} summarizes the evaluation results in NDS and the TP IoU and EC-IoU of three object classes, ``car," ``truck," and ``pedestrian." The results show that both IoU and EC-IoU are positively correlated with NDS. Notably, from all models, we see a substantial drop in EC-IoU, compared to IoU, for ``truck". This implies that EC-IoU offers an additional assessment dimension that signifies if a specific class is undermined in terms of safety, In this case, ``turck" is the outstanding class, potentially due to its larger size being difficult to properly cover from the ego's angle. Accordingly, one should take mitigations, e.g., by enlarging the predictions with a certain factor~\cite{schuster2022unaligned,degrancey2022object}.

\subsubsection{KITTI}

Different from the neScenes benchmark, KITTI first matches the sets of predictions and ground truths using IoU with predefined thresholds and then calculates the standard AP by counting true-positive and false-negative predictions~\cite{geiger2012are}. We supplement IoU with EC-IoU in this process, resulting in a parallel EC-AP metric. Essentially, by doing so, predictions not achieving sufficient EC-IoU will be directly removed, and the ones with higher EC-IoU will be preferred during matching.

We focus on two popular camera-based models that are available on the MMDetection3D platform~\cite{mmdet3d2020}, including SMOKE~\cite{liu2020smoke} and PGD~\cite{wang2021probabilistic}. The results are organized in the upper two rows of Table~\ref{tab:kitti}. Due to the small number of samples, we do not report the ``cyclist" class. For ``pedestrian" and ``car," we follow the official protocol to use strict thresholds (0.5 and 0.7, respectively) and report AP40 of the moderate category~\cite{geiger2012are}.

From Table~\ref{tab:kitti}, we see that the more advanced PGD indeed achieves higher scores than SMOKE mostly. However, for the ``car" class and the overall mAP, it has a lower EC-AP score, as marked in italic font. This indicates that while PGD can generally locate objects better than SMOKE, it does not necessarily place its predictions ahead of the objects from the ego's point of view. Observing the phenomenon, we now perform fine-tuning on PGD and discuss the result in the following section.

\subsection{Real-World Object Detector Fine-Tuning}
\label{subsec:detector_finetuning}

For fine-tuning the PGD object detector, we apply the EC-IoU loss function, $L_{\textsf{EC}\text{-}\textsf{IoU}}$. As mentioned in Sec.~\ref{subsec:opt_sim}, we also employ its counterpart $L_{\textsf{IoU}}$ to benchmark their performance. We do not use the DIoU or EIoU variants, as fine-tuning starts from a relatively well-performing model, foregoing the necessity of strong regularization signals (as well as potential noises). 

For implementation, since the original prediction head of the PGD object detector produces an image-based representation of 3D bounding boxes, we append to it a transformation function and attain the typical 3D representation for all predictions, i.e., $(x_\mathbf{P}, y_\mathbf{P}, z_\mathbf{P}, l_\mathbf{P}, w_\mathbf{P}, h_\mathbf{P}, \theta_\mathbf{P})$. Then, as introduced in Sec.~\ref{subsec:complexity}, the 3D version of the EC-IoU can be computed between the predictions and ground-truth targets. Our optimization configurations (e.g., the learning rate and data augmentation policy) follow the original PGD work~\cite{wang2021probabilistic}. We conduct the fine-tuning for 6 epochs with a batch size of 6 samples on an Nvidia RTX A6000 GPU.

Tab.~\ref{tab:kitti} presents the best result in 10 runs. $L_{\textsf{EC}\text{-}\textsf{IoU}}$ demonstrates more substantial increases in most numbers, including the standard mAP. In particular, it attains much higher EC-AP for the ``car" class as well as on average. Fig.~\ref{fig:kitti_results} provides qualitative results from two scenes. For most ground-truth objects, the blue boxes given by $L_{\textsf{EC}\text{-}\textsf{IoU}}$ either cover the objects more properly from the ego's perspective or achieve better alignment towards the objects.

Nonetheless, it is also noted in Tab.~\ref{tab:kitti} that for ``pedestrian," $L_{\textsf{EC}\text{-}\textsf{IoU}}$ delivers a lower EC-AP than the baseline, similar to $L_{\textsf{IoU}}$ in both standard AP and EC-AP. Such a performance drop is speculated to occur due to the smaller object size and a smaller number of instances in the dataset (4487 pedestrians vs. 28742 cars). Moreover, we recall that, based on the weighting formulation in Eq.~\eqref{eq:weighting}, EC-IoU tends to care more about close objects and falls back to IoU for faraway ones. Meanwhile, in real-world datasets, it is less common to see close pedestrians than close cars, leading to the performance difference between the two classes. For future improvement, EC-IoU may be combined with importance weighting schemes at the object level to emphasize specific classes or certain distance ranges (e.g.,~\cite{cheng2020safety,lyssenko2022towards}).

\section{Conclusion}

In this work, we developed EC-IoU, a safety-driven assessment approach that extends the existing ones such as IoU. Given an object, when two predictions~$\mathbf{P}_1$ and $\mathbf{P}_2$ share the same IoU value, $\mathbf{P}_1$ with a higher EC-IoU value implies that the predicted location is slightly closer to the ego vehicle, thereby preventing the downstream planning algorithm from safety surprises (i.e., an object is closer than expected). We demonstrated the mathematical properties of EC-IoU and, due to the intractability of a closed-form computation, proposed a precise and efficient approximation based on Mean Value Theorem. We conducted experiments with simulation and the representative nuScenes and KITTI datasets, confirming our proposal's advantage in explicit safety characterization. On a broader scope, our work aligns with recent investigations and flags a novel attempt to incorporate safety principles into the design and evaluation of learning-enabled algorithms. Moreover, it offers many avenues for exploration. Apart from further evaluation, we consider making~$\alpha$, the parameter controlling EC-IoU's weighting mechanism, more adaptive according to object distances or time-to-collision. Another interesting direction is to use EC-IoU as an indicator for online run-time monitoring.

\balance
\bibliographystyle{IEEEtran}
\bibliography{ref}

\begin{thebibliography}{10}
\providecommand{\url}[1]{#1}
\csname url@rmstyle\endcsname
\providecommand{\newblock}{\relax}
\providecommand{\bibinfo}[2]{#2}
\providecommand\BIBentrySTDinterwordspacing{\spaceskip=0pt\relax}
\providecommand\BIBentryALTinterwordstretchfactor{4}
\providecommand\BIBentryALTinterwordspacing{\spaceskip=\fontdimen2\font plus
\BIBentryALTinterwordstretchfactor\fontdimen3\font minus
  \fontdimen4\font\relax}
\providecommand\BIBforeignlanguage[2]{{%
\expandafter\ifx\csname l@#1\endcsname\relax
\typeout{** WARNING: IEEEtran.bst: No hyphenation pattern has been}%
\typeout{** loaded for the language `#1'. Using the pattern for}%
\typeout{** the default language instead.}%
\else
\language=\csname l@#1\endcsname
\fi
#2}}

\bibitem{wu2020recent}
X.~Wu, D.~Sahoo, and S.~C. Hoi, ``Recent advances in deep learning for object
  detection,'' \emph{Neurocomputing}, 2020.

\bibitem{iso21448}
``{ISO} 21448:2022 {R}oad vehicles - {S}afety of the intended functionality,''
  \url{https://www.iso.org/standard/77490.html}, 2023.

\bibitem{ul4600}
``{ANSI/UL} 4600:2023 {S}tandard for safety for the evaluation of autonomous
  products,''
  \url{https://ulse.org/ul-standards-engagement/autonomous-vehicle-technology},
  2023.

\bibitem{padilla2020survey}
R.~Padilla, S.~L. Netto, and E.~A.~B. da~Silva, ``A survey on performance
  metrics for object-detection algorithms,'' in \emph{IWSSIP}, 2020.

\bibitem{he2022alphaiou}
J.~He, S.~Erfani, X.~Ma, J.~Bailey, Y.~Chi, and X.-S. Hua, ``Alpha-{IoU}: {A}
  family of power intersection over union losses for bounding box regression,''
  in \emph{NeurIPS}, 2021.

\bibitem{caesar2020nuscenes}
H.~Caesar, V.~Bankiti, A.~H. Lang, S.~Vora, V.~E. Liong, Q.~Xu, A.~Krishnan,
  Y.~Pan, G.~Baldan, and O.~Beijbom, ``nu{S}cenes: {A} multimodal dataset for
  autonomous driving,'' in \emph{CVPR}, 2020.

\bibitem{geiger2012are}
A.~Geiger, P.~Lenz, and R.~Urtasun, ``Are we ready for autonomous driving?
  {T}he {KITTI} vision benchmark suite,'' in \emph{CVPR}, 2012.

\bibitem{mmdet3d2020}
{MMDetection3D Contributors}, ``{MMDetection3D: OpenMMLab} next-generation
  platform for general {3D} object detection,''
  \url{https://github.com/open-mmlab/mmdetection3d}, 2020.

\bibitem{ma2023object}
X.~Ma, W.~Ouyang, A.~Simonelli, and E.~Ricci, ``{3D} object detection from
  images for autonomous driving: {A} survey,'' \emph{TPAMI}, 2023.

\bibitem{mao2023object}
J.~Mao, S.~Shi, X.~Wang, and H.~Li, ``{3D} object detection for autonomous
  driving: {A} comprehensive survey,'' \emph{IJCV}, 2023.

\bibitem{lowe1999object}
D.~Lowe, ``Object recognition from local scale-invariant features,'' in
  \emph{ICCV}, 1999.

\bibitem{viola2001rapid}
P.~Viola and M.~Jones, ``Rapid object detection using a boosted cascade of
  simple features,'' in \emph{CVPR}, 2001.

\bibitem{ren2015faster}
S.~Ren, K.~He, R.~Girshick, and J.~Sun, ``{Faster R-CNN}: {T}owards real-time
  object detection with region proposal networks,'' in \emph{NeurIPS}, 2015.

\bibitem{liu2016ssd}
W.~Liu, D.~Anguelov, D.~Erhan, C.~Szegedy, S.~Reed, C.-Y. Fu, and A.~C. Berg,
  ``{SSD}: {S}ingle shot multibox detector,'' in \emph{ECCV}, 2016.

\bibitem{yu2016unitbox}
J.~Yu, Y.~Jiang, Z.~Wang, Z.~Cao, and T.~Huang, ``{UnitBox}: {A}n advanced
  object detection network,'' in \emph{{MM}}, 2016.

\bibitem{rezatofighi2018generalized}
H.~Rezatofighi, N.~Tsoi, J.~Gwak, A.~Sadeghian, I.~Reid, and S.~Savarese,
  ``Generalized {I}ntersection over {U}nion: {A} metric and a loss for bounding
  box regression,'' in \emph{CVPR}, 2019.

\bibitem{zheng2019distanceiou}
Z.~Zheng, P.~Wang, W.~Liu, J.~Li, R.~Ye, and D.~Ren, ``Distance-{IoU} loss:
  {F}aster and better learning for bounding box regression,'' in \emph{{AAAI}},
  2020.

\bibitem{zhang2022focal}
Y.-F. Zhang, W.~Ren, Z.~Zhang, Z.~Jia, L.~Wang, and T.~Tan, ``Focal and
  efficient {IoU} loss for accurate bounding box regression,''
  \emph{Neurocomputing}, 2022.

\bibitem{lin2017focal}
T.-Y. Lin, P.~Goyal, R.~Girshick, K.~He, and P.~Dollar, ``Focal loss for dense
  object detection,'' in \emph{ICCV}, 2017.

\bibitem{hung2022let3dap}
W.-C. Hung, H.~Kretzschmar, V.~Casser, J.-J. Hwang, and D.~Anguelov,
  ``{LET-3D-AP}: {L}ongitudinal error tolerant {3D} average precision for
  camera-only {3D} detection,'' 2022.

\bibitem{deng2021revisiting}
B.~Deng, C.~R. Qi, M.~Najibi, T.~Funkhouser, Y.~Zhou, and D.~Anguelov,
  ``Revisiting {3D} object detection from an egocentric perspective,'' in
  \emph{NeurIPS}, 2021.

\bibitem{volk2020comprehensive}
G.~Volk, J.~Gamerdinger, A.~v. Bernuth, and O.~Bringmann, ``A comprehensive
  safety metric to evaluate perception in autonomous systems,'' in \emph{ITSC},
  2020.

\bibitem{cheng2020safety}
C.-H. Cheng, ``Safety-aware hardening of {3D} object detection neural network
  systems,'' in \emph{SafeComp}, 2020.

\bibitem{lyssenko2022towards}
M.~Lyssenko, C.~Gladisch, C.~Heinzemann, M.~Woehrle, and R.~Triebel, ``Towards
  safety-aware pedestrian detection in autonomous systems,'' in \emph{IROS},
  2022.

\bibitem{schuster2022unaligned}
T.~Schuster, E.~Seferis, S.~Burton, and C.-H. Cheng, ``Unaligned but safe --
  {F}ormally compensating performance limitations for imprecise 2d object
  detection,'' in \emph{SafeComp}, 2022.

\bibitem{degrancey2022object}
F.~de~Grancey, J.-L. Adam, L.~Alecu, S.~Gerchinovitz, F.~Mamalet, and
  D.~Vigouroux, ``Object detection with probabilistic guarantees: {A} conformal
  prediction approach,'' in \emph{SafeComp Worshops}, 2022.

\bibitem{liu2022bevfusion}
Z.~Liu, H.~Tang, A.~Amini, X.~Yang, H.~Mao, D.~Rus, and S.~Han, ``{BEVFusion}:
  {M}ulti-task multi-sensor fusion with unified bird's-eye view
  representation,'' in \emph{ICRA}, 2023.

\bibitem{gilles2023shapely}
S.~Gillies, C.~van~der Wel, J.~Van~den Bossche, M.~W. Taves, J.~Arnott, B.~C.
  Ward, and {others}, ``{Shapely},'' \url{https://github.com/shapely/shapely},
  2023.

\bibitem{berg2008computational}
M.~Berg, O.~Cheong, M.~Kreveld, and M.~Overmars, \emph{Computational Geometry:
  {A}lgorithms and applications}.\hskip 1em plus 0.5em minus 0.4em\relax
  Springer, 2008.

\bibitem{wang2021probabilistic}
T.~Wang, X.~Zhu, J.~Pang, and D.~Lin, ``Probabilistic and geometric depth:
  {D}etecting objects in perspective,'' in \emph{CoRL}, 2021.

\bibitem{zhu2020ssn}
X.~Zhu, Y.~Ma, T.~Wang, Y.~Xu, J.~Shi, and D.~Lin, ``{SSN: Shape} signature
  networks for multi-class object detection from point clouds,'' in
  \emph{ECCV}, 2020.

\bibitem{yin2021center}
T.~Yin, X.~Zhou, and P.~Kr{\"a}henb{\"u}hl, ``Center-based {3D} object
  detection and tracking,'' \emph{CVPR}, 2021.

\bibitem{wang2021fcos3d}
T.~Wang, X.~Zhu, J.~Pang, and D.~Lin, ``{FCOS3D: Fully} convolutional one-stage
  monocular {3D} object detection,'' in \emph{ICCVW}, 2021.

\bibitem{liu2020smoke}
Z.~Liu, Z.~Wu, and R.~T\'oth, ``{SMOKE}: {S}ingle-stage monocular 3{D} object
  detection via keypoint estimation,'' 2020.

\end{thebibliography}

\end{document}